\documentclass{article}

\usepackage[T1]{fontenc}
\usepackage[utf8]{inputenc}
\usepackage[a4paper,left=1.8cm,right=1.8cm,top=2cm,bottom=2cm]{geometry}
\usepackage{amsmath, amssymb, amsfonts, amsthm}
\usepackage{mathtools}
\usepackage{graphicx}
\usepackage{booktabs}
\usepackage{multirow}
\usepackage{siunitx}
\usepackage{hyperref}
\usepackage{xcolor}
\usepackage{enumitem}

\newcommand{\noise}{\eta}

\newcommand{\radon}{\mathcal{R}}

\newcommand{\ee}{e}
\newcommand{\signal}{x}
\newcommand{\data}{y}

\newcommand{\So}{\mathcal{S}}
\newcommand{\Mo}{\mathcal{M}}
\newcommand{\derivative}{\mathcal{U}}

\newcommand{\forward}{\mathcal{A}}
\newcommand{\rec}{\mathcal{B}}

\newcommand{\eexp}{\exp}
\newcommand{\llog}{\log}

\newcommand{\R}{\mathbb{R}}

\newcommand{\E}{\mathbb{E}}
\newcommand{\loss}{\mathcal{L}}

\newcommand\set[1]{\{#1\}}

\DeclareMathOperator{\psnr}{PSNR}

\newtheorem{theorem}{Theorem}
\newtheorem{definition}{Definition}
\newtheorem{prop}{Proposition}

\newenvironment{myitem}
  {\begin{itemize}[leftmargin=2em,itemsep=0.5em]}
  {\end{itemize}}

\newenvironment{tritemize}
  {\begin{itemize}[leftmargin=2em,itemsep=0.4em,label=$\blacktriangleright$]}
  {\end{itemize}}

\title{SPLIT: Self-supervised Partitioning for Learned Inversion in Nonlinear Tomography}
\author{Markus Haltmeier, Lukas Neumann, Nadja Gruber, Gyeongha Hwang}
\date{}

\begin{document}

\maketitle

\begin{abstract}
Machine learning has achieved impressive performance in tomographic reconstruction, but supervised training requires paired measurements and ground-truth images that are often unavailable. This has motivated self-supervised approaches, which have primarily addressed denoising and, more recently, linear inverse problems. We address nonlinear inverse problems and introduce SPLIT (Self-supervised Partitioning for Learned Inversion in Nonlinear Tomography), a self-supervised machine-learning framework for reconstructing images from nonlinear, incomplete, and noisy projection data without any samples of ground-truth images. SPLIT enforces cross-partition consistency and measurement-domain fidelity while exploiting complementary information across multiple partitions. Our main theoretical result shows that, under mild conditions, the proposed self-supervised objective is equivalent to its supervised counterpart in expectation. We regularize training with an automatic stopping rule that halts optimization when a no-reference image-quality surrogate saturates. As a concrete application, we derive SPLIT variants for multispectral computed tomography. Experiments on sparse-view acquisitions demonstrate high reconstruction quality and robustness to noise, surpassing classical iterative reconstruction and recent self-supervised baselines.
\end{abstract}



\section{Introduction}

Accurate image reconstruction is a fundamental task in computed tomography (CT) and elsewhere, critical for medical diagnosis and industrial applications. In this paper, we study general tomographic reconstruction problems of the form of a nonlinear inverse problem
\begin{equation}\label{eq:ip}
  \data = \forward(\signal) + \noise,
\end{equation}
where $\forward \colon X \to Y$ is the (possibly) nonlinear forward operator between finite-dimensional Hilbert spaces, $\signal \in X$ the unknown image, and $\data \in Y$ the observed measurement data. The term $\noise \in Y$ models potentially signal-dependent measurement noise; in practice, photon statistics are often Poisson and electronic noise is additive. While this work is motivated by and tested on multispectral CT (MSCT), the proposed reconstruction framework is general, with focus on non-linear tomography.

Classical linear tomography relies on inverting the linear Radon transform, which models internal structure as a scalar-valued discrete attenuation map $\mu \in \mathbb{R}^N$ recovered from a noisy version of its sinogram $\radon \mu \in \mathbb{R}^P$, where $\radon$ denotes the discrete Radon transform. While effective in many settings, this linear model neglects crucial physical realities: the polychromatic nature of X-ray sources and the energy-dependent attenuation of materials. Real samples are better described by a family of attenuation maps $\mu(\ee) \in \mathbb{R}^N$, parameterized by photon energy $\ee \in (0,\infty)$. Using a single energy bin merges these distinct maps into a single projection, causing non-uniqueness and severe artifacts such as beam hardening. Although iterative and analytic correction methods exist \cite{mcdavid1975spectral,kiss2023beam,pan2008anniversary,herman1979correction,van2011iterative,rigaud2017analytical}, they only partially mitigate these issues for a single X-ray source. Multispectral CT (MSCT) addresses these challenges by acquiring projection data across multiple energy bands, enabling reconstruction of energy-dependent attenuation maps. This yields a fundamentally nonlinear and more complex inverse problem \cite{kazantsev2018joint,rigie2015joint,hu2019nonlinear,arridge2021overview,Mory_2018,heismann2009quantitative,maass2009image,barber2016algorithm,bevilacqua2023regularized}. 

The MSCT image reconstruction problem can be written in the form \eqref{eq:ip}, where the unknown $\signal = (\signal[1], \dots, \signal[M]) \in X = \mathbb{R}^{N \times M}$ consists of a stack of material density images and the data $\data = (\data[1], \dots, \data[B]) \in Y = \mathbb{R}^{P \times B}$ are multispectral measurements across $B$ energy channels. The forward map $\forward = \Phi \circ \radon$ combines the Radon transform $\radon \colon \mathbb{R}^{N \times M} \to \mathbb{R}^{P \times M}$ with a pointwise nonlinearity $\Phi \colon \mathbb{R}^M \to \mathbb{R}^B$ that maps, per ray, $M$ line integrals to $B$ spectral measurements. Classical analytic and iterative methods, as well as supervised learning, have been applied to \eqref{eq:ip}, but supervised approaches require paired $(\signal,\data)$, which are difficult to obtain. Self-supervised methods \cite{batson2019noise2self,hendriksen2020noise2inverse,gruber2024sparse2inverse,schut2025equivariance2inverse,unal2024proj2proj,yaman2020self,gruber2025noisier2inverse,chen2022robust} alleviate this by training only from measurements $\data$. In particular, relatives of our approach have shown success in denoising \cite{batson2019noise2self}, full-data inverse problems \cite{hendriksen2020noise2inverse}, and linear settings with fixed, undersampled operators \cite{gruber2024sparse2inverse,unal2024proj2proj}. However, to the best of our knowledge, no self-supervised approaches exist for nonlinear tomography.

In this paper, we propose a self-supervised framework for tomographic inverse problems that addresses nonlinear forward models, undersampling, and noise. We introduce SPLIT (Self-supervised Partitioning for Learned Inversion in Nonlinear Tomography), which allows multiple data splits. SPLIT partitions the data domain, targets reconstruction in image space, and employs a loss in the measurement domain. In its single-partition version, similar to Sparse2Inverse \cite{gruber2024sparse2inverse}, SPLIT uses one partition along the angular dimension. To adequately account for noise and undersampling, we also introduce Double-SPLIT variants, which introduce a second partition along the detector (projection) dimension. We consider these methods within the SPLIT formulation for general nonlinear problems, potentially addressing multiple challenges present in the data. To the best of our knowledge, such multi-partition strategies have not been applied in self-supervised reconstruction before. We implement SPLIT for MSCT with a partition size of two in each dimension and demonstrate improvements over the self-supervised baseline. While our primary target is sparse-angle MSCT, these approaches apply to any linear or non-linear forward model that admits multiple data partitions.

\section{SPLIT Framework}

In this section we introduce the proposed SPLIT technique. We consider general inverse problems \eqref{eq:ip}, where $\forward \colon \R^{N} \to \R^{P}$ is a deterministic possibly nonlinear forward map. We assume the unknown $\signal$ and the noise $\noise$ are random and follow unknown distributions $p_\signal$ and $p_\noise$. The proposed SPLIT technique is inspired by Sparse2Inverse; however, it is applicable to nonlinear problems. Moreover, opposed to Sparse2Inverse we allow multiple partitions, which, to the best of our knowledge, has not been explored in the literature on self-supervised imaging.

\subsection{Self-supervised imaging}
\label{ssec:self}

Recently, self-supervised methods have been developed for linear inverse problems: Noise2Self for denoising, Noise2Inverse for full-data inverse problems, and Sparse2Inverse for limited-data inverse problems. Throughout this subsection we assume that $x \in X = \R^N$ and $y \in Y = \R^P$ are random variables with joint distribution $P_{x,y}$. For any $z \in \R^P$ and any $\Omega \subseteq \{1,\dots,P\}$, denote
\[
z_\Omega = (z_p)_{p \in \Omega}, \qquad z_\Omega^c = (z_p)_{p \notin \Omega}.
\]
Self-supervised imaging methods are based on diagonal-free maps with respect to certain partitions of $\{1,\dots,P\}$ as outlined in the following.

\begin{myitem}
 
\item \textsc{Diagonal-free maps:}
The crucial concept of self-supervised methods is diagonal-free nonlinear maps. Diagonal-free nonlinear maps, termed $\mathcal{J}$-invariant in the Noise2Self framework \cite{batson2019noise2self} for self-supervised denoising, formalize the idea that each output coordinate must not depend on the corresponding input coordinate. Closely related structural constraints on operators appear in functional analysis and numerical analysis through off-diagonal or neighbor-only operators and sparsity patterns \cite{davis2006direct}, as well as in dynamical systems through coupling structures and coordinate-wise dependence restrictions in discrete-time maps \cite{brin2002introduction}.

\begin{definition}[Diagonal free nonlinear maps]
Let $g \colon \mathbb{R}^P \to \mathbb{R}^P$, 
$\Lambda \subseteq \{1, \dots, P\}$ and  $\mathcal{P}$ be a partition of $\{1, \dots, P\}$. The function $g$ is called \emph{$\Lambda$-self-invariant} if 
\(
g(z)_\Lambda \text{ depends only on } z_\Lambda^c
\). It  is called   diagonal-free (with respect to $\mathcal{P}$) if it is $\Omega$-self-invariant for every $\Omega \in \mathcal{P}$.
\end{definition}

As a simple example consider $P = 3$ and a nonlinear map $g \colon \mathbb{R}^3 \to \mathbb{R}^3$ defined by
\[
g(z_1,z_2,z_3) = 
\begin{pmatrix}
\sin(z_2 + z_3) \\
z_1 + \exp(z_3) \\
\cos(z_1 - z_2)
\end{pmatrix}.
\]
Then, with the partition $\mathcal{P} = \{\{1\}, \{2\}, \{3\}\}$, $g$ is diagonal-free: each component $g_i$ depends only on the other two components with $z_j$ with $j \in \set{1,2,3} \setminus \set{i}$ and not on $z_i$ itself.  More generally, for any functions $f, h_\Omega \colon \R^P \to \R^P$ and any partition $\mathcal{P}$, a diagonal-free map $g \colon \mathbb{R}^P \to \mathbb{R}^P$ with respect to $\mathcal{P}$ can be defined by    
\begin{equation} \label{eq:unetJ}   
\forall \Omega \in \mathcal{P} \colon \qquad 
\big(g(y)\big)_\Omega \coloneqq \big(f\big( \Mo_{\Omega}\, h_\Omega(\Mo_\Omega^c y) + \Mo_\Omega^c\, y \big)\big)_\Omega \,,
\end{equation}
where $\Mo_{\Omega}, \Mo_{\Omega }^c \in \{0,1\}^{P \times P}$ are diagonal masking matrices that zero out entries outside or inside $\Omega$, respectively. Then, by construction, $g$ is diagonal-free with respect to $\mathcal{P}$. In typical applications, $\Mo_{\Omega}\, h_\Omega(\Mo_\Omega^c\, y)$ is either the zero function or a classical interpolation (e.g., local averaging) from the complement $\Omega^c$ to the target region $\Omega$, and $f$ is a neural network whose diagonal-free variant is given by $g$.

 \item \textsc{Noise2Self:} Consider first the denoising problem, which is  \eqref{eq:ip} with $\forward = \mathrm{id}$ and $P=N$. The reconstruction function $g$ is then a denoiser, ideally minimizing the supervised risk $\E \|g(y) - x\|^2$. Instead, Noise2Self minimizes the self-supervised loss $\E \|g(y) - y\|^2$ over diagonal-free functions. As shown in the seminal work \cite{batson2019noise2self}, this is justified by the following result.

\begin{prop}[Noise2Self]\label{prop:n2s}
Let $\noise$ be pairwise independent with zero mean and let $g \colon \R^N \to \R^N$ be diagonal-free with respect to $\mathcal{P}$. Then,  \(
\E \|g(y) - y\|^2 \;=\; \E \|g(y) - x\|^2 \;+\; \E \|\noise\|^2 \).
\end{prop}

Based on Prop.~\ref{prop:n2s}, the risk $\E \|g(y) - x\|^2$ is minimized by the same $\mathcal{P}$-invariant function as the self-supervised risk $\E \|g(y) - y\|^2$. As demonstrated in \cite{batson2019noise2self}, the performance of these estimators is very close to supervised ones.

\item \textsc{Noise2Inverse:}
Next, consider linear inverse problems written in the form \eqref{eq:ip} with a linear map $\forward$. Noise2Self can be applied either in the measurement domain for denoising $\data$ followed by image formation, or in the reconstruction domain for denoising an initial reconstruction $\rec \data$. Both approaches suffer from a two-step nature and further drawbacks, such as correlation of the noise in the reconstruction domain $X$ and less informative   prior in the measurement domain $Y$. To address these issues, \cite{hendriksen2020noise2inverse} proposes Noise2Inverse, specifically where $\forward$ is the full-view Radon transform. Here, a partition $\mathcal{P}$ is used in the projection domain where the noise is statistically independent, and a network is trained in the reconstruction domain. For any subset $\Omega \in \mathcal{P}$, Noise2Inverse considers linear reconstructions $\rec_\Omega, \rec_\Omega^c$ using data $\data_\Omega, \data_\Omega^c$. Then a network $f \colon \R^N \to \R^N$ is constructed by minimizing
\(
\E
\|
f(\rec_{\Omega}^c \data_{\Omega}^c )
-
\rec_{\Omega} \data_{\Omega} 
\|^2
\).
Noise2Inverse is supported by the following result derived in \cite{hendriksen2020noise2inverse}.

\begin{prop}[Noise2Inverse]\label{prop:n2i}
Suppose $\noise$ has zero mean and is pairwise independent. Then, for any $f \colon \R^N \to \R^N$, we have           
\(
\E \|
f(\rec_{\Omega}^c \data_{\Omega}^c )
-
\rec_{\Omega} \data_{\Omega}
\|^2
=
\E \|
f(\rec_{\Omega}^c \data_{\Omega}^c)
-
\rec_{\Omega}\forward_{\Omega}\signal
\|^2
+ \E \| \rec_{\Omega} \noise \|^2
\).
\end{prop}

Based on Prop.~\ref{prop:n2i}, the supervised risk
\(
\E \|
f(\rec_{\Omega}^c \data_{\Omega}^c)
-
\rec_{\Omega} \forward_{\Omega}\signal
\|^2
\)
and the self-supervised risk
\(
\E \|
f(\rec_{\Omega}^c \data_{\Omega}^c)
-
\rec_{\Omega} \data_{\Omega}
\|^2
\)
share the same minimizer. After computing the minimizer within a chosen function class, inference is  done by
\(
\rec(\data)
\coloneq
|\mathcal{P}|^{-1}
\sum_{\Omega \in \mathcal{P}}
f(\rec_{\Omega}^c \data_{\Omega}^c)
\).

\item \textsc{Sparse2Inverse:}
Noise2Inverse aims to remove measurement noise but is not designed to remove artifacts resulting from undersampling in the case of sparse data, when $\forward$ has a large null space. The reason is that the targets in this framework are defined using linear reconstruction methods and thus still fix the null-space component shared by any of the maps $\rec_{\Omega}\forward_{\Omega}$. To overcome this issue, in \cite{gruber2024sparse2inverse} we presented a method that uses a network in reconstruction domain but with a loss in the measurement domain (also see \cite{unal2024proj2proj}).
More specifically, Sparse2Inverse minimizes
\(
\E \| \forward_{\Omega} f(\rec_{\Omega}^c \data_{\Omega}^c) - \data_{\Omega} \|^2
\)
based on the following result.

\begin{prop}[Sparse2Inverse]\label{prop:s2i}
Suppose $\noise$ is pairwise independent with zero mean.
Then for any $f \colon \R^N \to \R^N$, we have           
\(
\E \|
\forward_{\Omega} f(\rec_{\Omega}^c \data_{\Omega}^c)
-
\data_{\Omega}
\|^2
=
\E \|
\forward_{\Omega} f(\rec_{\Omega}^c \data_{\Omega}^c)
-
\forward_{\Omega}\signal
\|^2
+ \E \| \noise \|^2
\).
\end{prop}

Based on Prop.~\ref{prop:s2i},
\(
\E \|
\forward_{\Omega} f(\rec_{\Omega}^c \data_{\Omega}^c)
-
\forward_{\Omega}\signal
\|^2
\)
and
\(
\E \|
\forward_{\Omega} f(\rec_{\Omega}^c \data_{\Omega}^c)
-
\data_{\Omega}
\|^2
\)
share the same minimizer. After computing the minimizer within a chosen function class, inference can be done by
\(
\rec(\data)
\coloneq
|\mathcal{P}|^{-1}
\sum_{\Omega \in \mathcal{P}}
f\bigl(\rec_{\Omega}^c \data_{\Omega}^c\bigr)
\).
The crucial difference compared to the loss in Noise2Inverse is the freedom of the network to operate in the null space of $\forward$. This is a prerequisite for the network to effectively remove undersampling artifacts inherent in $\rec_{\Omega}\forward_{\Omega}$. In numerical reconstruction, such a behavior in indeed observed \cite{gruber2024sparse2inverse,schut2025equivariance2inverse}.

\end{myitem}

\subsection{Proposed SPLIT}

The SPLIT technique consists of several elements that we describe below. It is applicable to any linear or nonlinear inverse problem of the form \eqref{eq:ip}. At its core, it uses multiple partitions in the measurement domain with the intention of addressing different challenges in \eqref{eq:ip}, for example, missing data and denoising. We first derive it in a general framework for nonlinear problems with an arbitrary number of partitions and then specialize it to Double-SPLIT for MSCT.

\begin{myitem}

\item \textsc{Multi-partitioning:}
In contrast to existing self-supervised approaches that use a single partition, we include multiple partitions of $\{1,\dots,P\}$ to potentially account for different tasks. The partitions are denoted by $\mathcal{P}_i$ for $i = 1,\dots,K$ with elements $\Omega_{i,j} \subseteq \{1,\dots,P\}$ for $j=1,\dots,|\mathcal{P}_i|$. We denote the data and forward operator restricted to $\Omega_{i,j}$ and its complement $\Omega_{i,j}^c$ by
\begin{align} \label{eq:split1}
 \data_{i,j} &\coloneqq \data|_{\Omega_{i,j}} \in \mathbb{R}^{| \Omega_{i,j} |} \,, \\
 \label{eq:split2}
 \data_{i,j}^c &\coloneqq \data|_{\Omega_{i,j}^c} \in \mathbb{R}^{| \Omega_{i,j}^c |} \,, \\
 \label{eq:split3}
 \forward_{i,j} \signal 
&\coloneqq (\forward \signal)|_{\Omega_{i,j}} \,, \\
 \label{eq:split4}
  \forward_{i,j}^c \signal 
  &\coloneqq (\forward \signal)|_{\Omega_{i,j}^c} \,.
\end{align}
This yields measurement-domain splitting similar to Noise2Inverse and Sparse2Inverse, now generalized to multiple partitions.

\item \textsc{Partial reconstructions:}
For each partition $\mathcal{P}_i$ and subset $\Omega_{i,j} \in \mathcal{P}_i$, we consider fixed reconstruction operators
\begin{align} \label{eq:split5}
& \rec_{i,j} \colon \mathbb{R}^{| \Omega_{i,j} |} \to \mathbb{R}^N, \\
\label{eq:split6}
& \rec_{i,j}^c \colon \mathbb{R}^{| \Omega_{i,j}^c |} \to \mathbb{R}^N.
\end{align}
Reconstruction pairs $\rec_{i,j}^c \data_{i,j}^c$ and $\rec_{i,j} \data_{i,j}$ then provide data allowing self-supervised training. In the spirit of Noise2Inverse, one would use these for denoising with the norm in the reconstruction space $\R^N$, thereby fixing the null space component of the solution selection. In the spirit of Sparse2Inverse, one supervises using a data-space loss, thus allowing more freedom for the network to select solutions with  modified null space component. Further note that we do not aim for $\rec_{i,j}$ and $\rec_{i,j}^c$ to be denoising. Ideally, they are right inverses with
$\forward_{i,j}\rec_{i,j}\forward_{i,j} = \forward_{i,j}$
and
$\forward_{i,j}^c\rec_{i,j}^c\forward_{i,j}^c = \forward_{i,j}^c$.
Due to the limited-data nature of $\forward_{i,j}$ and $\forward_{i,j}^c$, there are plenty of such maps.

\item \textsc{Self-supervised loss:}  
A map $f \colon \mathbb{R}^N \to \mathbb{R}^N$, chosen from a certain family, is trained to map the auxiliary input $\rec_{i,j}^c \data_{i,j}^c$ to the auxiliary output $\rec_{i,j} \data_{i,j}$. For that purpose we  consider
\begin{align} \label{eq:multiX}
\loss^{K}_X(f)
&=
\frac{1}{K} \sum_{i=1}^{K} 
\sum_{j=1}^{|\mathcal{P}_i|}
\left\| f\big( \rec_{i,j}^c(\data_{i,j}^c) \big) - \rec_{i,j}(\data_{i,j}) \right\|^2,
\\ \label{eq:multiY}
\loss^{K}_Y(f)
&=
\frac{1}{K} \sum_{i=1}^{K} 
\sum_{j=1}^{|\mathcal{P}_i|}
\left\| \forward_{i,j}\!\left( f\big( \rec_{i,j}^c(\data_{i,j}^c) \big) \right) - \data_{i,j} \right\|^2.
\end{align}
Similar to Sparse2Inverse, \eqref{eq:multiY} is a self-supervised loss in measurement space, whereas \eqref{eq:multiX} is a self-supervised loss in image space. While \eqref{eq:multiY} is our proposal, we will test both loss terms in the context of MSCT.

\item \textsc{SPLIT-inference:}  
For inference, in both cases we use the reconstruction map
\begin{equation} \label{eq:multisplit-inf}
\rec(\data) \coloneq 
\frac{1}{K} \sum_{i=1}^{K}   
\frac{1}{|\mathcal{P}_i|} \sum_{j=1}^{|\mathcal{P}_i|}
f_\theta\big( \rec_{i,j}^c \data_{i,j}^c \big),
\end{equation}
i.e., the average of all individual reconstructions across partitions and subsets. For linear problems and for a single partition this reduces to the inference used in Sparse2Inverse and Noise2Inverse, respectively.
\end{myitem}

For a single partition ($K=1$), we refer to \eqref{eq:multiY}, \eqref{eq:multisplit-inf} as single-SPLIT, which is the straightforward extension of Sparse2Inverse to nonlinear problems. The use of multiple partitions naturally allows different splits for different tasks, such as denoising and missing-data recovery; the Double-SPLIT employs two complementary partitions, such as angular and detector domains in MSCT.

\subsection{Theoretical analysis}

Similar to single-split methods for linear inverse problems, we show that the self-supervised data-space loss coincides with its supervised counterpart.

\begin{theorem}[Multi-SPLIT] \label{thm:split}
Assume the measurement model \eqref{eq:ip} with a nonlinear forward map $\forward \colon \R^N \to \R^P$, where the noise $\noise$ has pixel-wise independent mean-zero entries. Let $(\mathcal{P}_i)_{i=1}^{K}$ be multiple data partitions and let the restricted data $\data_{i,j}, \data_{i,j}^c$ and partial forward and reconstruction maps $\forward_{i,j}, \forward_{i,j}^c, \rec_{i,j}, \rec_{i,j}^c$ be defined by \eqref{eq:split1}\textendash\eqref{eq:split6}. Then, for any $f \colon \R^N \to \R^N$, 
\begin{multline} \label{eq:split-thm}
\frac{1}{K} \sum_{i=1}^{K} 
\sum_{j=1}^{|\mathcal{P}_i|}
\E \Bigl[ \bigl\| \forward_{i,j} \big( f \big( \rec_{i,j}^c (\data_{i,j}^c) \big) \big) - \data_{i,j} \bigr\|^2 \Bigr]
\\ =
\frac{1}{K} \sum_{i=1}^{K}  
\sum_{j=1}^{|\mathcal{P}_i|}
\E \Bigl[  \bigl\| \forward_{i,j} \big( f \big( \rec_{i,j}^c (\data_{i,j}^c) \big) \big) - \forward_{i,j} \signal \bigr\|^2
\Bigr]
+ \E \bigl[  \|\noise\|^2 \bigr].
\end{multline}
\end{theorem}

\begin{proof}
It suffices to show the identity for a fixed partition and a selected set of the partition. To that end, fix $i \in \{1, \dots, K\}$ and $j \in \{1, \dots , |\mathcal{P}_i|\}$, and consider
\(
\E    \| \forward_{i,j}  ( f  ( \rec_{i,j}^c \data_{i,j}^c  ) ) - \data_{i,j}  \|^2  
\).
Using  $\data = \forward  (\signal) + \noise$ yields  
\begin{multline} \label{eq:split-aux}
\E 
\Bigl[ \big\|
\bigl(
\forward_{i,j} f(\rec_{i,j}^c \data_{i,j}^c) - \forward_{i,j} \signal \bigr)
- 
\noise_{i,j}
\big\|^2
\Bigr]
\\
= \E \Bigl[
 \bigl\| \forward_{i,j} f(\rec_{i,j}^c \data_{i,j}^c) - \forward_{i,j} \signal \bigr\|^2 \Bigr] 
+ 
2 \E \Bigl[
 \langle \forward_{i,j} f(\rec_{i,j}^c \data_{i,j}^c) - \forward_{i,j} \signal , \noise_{i,j} \rangle \Bigr] 
+ \E \bigl[
 \| \noise_{i,j} \|^2 \bigr].
\end{multline}
Now $\forward_{i,j} \signal$ is independent of $\noise_{i,j}$ and $\noise_{i,j}$ has zero mean, thus $\E \big[
 \langle \forward_{i,j} \signal , \noise_{i,j} \rangle \big] = 0$. Moreover, by construction, $\forward_{i,j} f(\rec_{i,j}^c \data_{i,j}^c)$ does not depend on $\noise_{i,j}$, and by the componentwise independence of the noise we have $\E \big[
 \langle \forward_{i,j} f(\rec_{i,j}^c \data_{i,j}^c) , \noise_{i,j} \rangle \big] = 0$. Combining these with \eqref{eq:split-aux} and summing over all $i,j$ gives \eqref{eq:split-thm}.
\end{proof}

We do not believe that a similar theorem can be derived for the Noise2Inverse-type loss~\eqref{eq:multiX}; we conjecture this is due to the necessarily nonlinear nature of the reconstruction maps $\rec_{i,j}$. Noise2Inverse for linear problems does not suffer from this limitation as long as linear reconstruction maps are used. 

Theorem \ref{thm:split} may be interpreted as follows: For any $i \in \set{1, \dots, N}$, the map $g_{i}(\data) \coloneqq ( (\forward_{i,j} \circ f \circ \rec_{i,j}^c)(\data_{i,j}^c) )_{j=1}^{|\mathcal{P}_i|}$ is diagonal-free with respect to the partition $\mathcal{P}_i$ in measurement space, enabling denoising there in the spirit of Noise2Self. The crucial ingredient for image reconstruction, however, is that the maps $g_{i}$ contain learned factors $f \circ \rec_{i,j}^c$ that map from measurement space to image space. The maps $f \circ \rec_{i,j}^c$ are then used to define a learned self-supervised reconstruction map via \eqref{eq:multisplit-inf}.

\section{Application to MSCT}
\label{sec:msct-realization}

In this section we present a concrete realization of SPLIT for MSCT using one and two partitions, referred to as Single-SPLIT and Double-SPLIT, respectively.  We present numerical details, experimental results, and comparisons with an iterative one-step algorithm and a self-supervised baseline.

\subsection{MSCT forward and inverse model}
\label{sec:msct}

We start with a brief derivation of the MSCT forward model, resulting in a nonlinear inverse problem of the form \eqref{eq:ip}. This includes modeling spectral measurements of the energy-dependent attenuation distribution and combining them with a material representation along the energy dimension. Additionally, we present an efficient one-step reconstruction algorithm for MSCT, used in the implementation of SPLIT.

\begin{myitem}\item \textsc{Spectral  Radon transform:}
Classical CT seeks to recover a scalar-valued discrete attenuation map $\mu \in \R^N$ for the patient or sample from observations of its (linear) Radon transform $\radon \mu \in \R^P$, obtained by comparing incoming and outgoing X-ray intensity. However, the energy-dependent absorption characteristics of real materials and the broadband nature of X-ray sources necessitate a family of attenuation maps $\mu(\ee,\cdot) \in \R^N$, parameterized by photon energy $\ee > 0$, for accurate modeling. Moreover, taking into account the source spectrum of emitted X-rays and the detector’s spectral response leads to measurements of the form
\begin{equation}\label{eq:data-mu}
  \data(\ell) \coloneq \int_{0}^{\infty} s(\ee)\,\exp\!\bigl(-\,(\radon \mu(\ee, \ell))\bigr)\,d\ee \,,
\end{equation}
where $s(\ee)$ is the effective spectrum (the product of the source spectrum and the detector spectral response), $\radon \mu(\ee, \ell)$ is the Radon transform of $\mu(\ee, \cdot)$ in the spatial variable, and $\ell$ is a ray  of integration. Due to spectral mixing across energies, recovering $(\mu(\ee,\cdot))_{\ee>0}$ from a single effective spectrum is severely ill-posed, leading to non-uniqueness and beam-hardening artifacts.
Multispectral CT (MSCT) mitigates these issues by acquiring projection data across multiple energy bands using distinct effective spectra $(s_b(\ee))_{b=1}^B$. Assuming a low-dimensional model along the energy axis (e.g., a material basis; see \eqref{eq:mu-ansatz} below), these measurements enable accurate reconstruction of energy-dependent attenuation.

\item \textsc{MSCT forward problem:}
We consider MSCT within the material decomposition framework in a discrete setting. The core assumption is that the energy-dependent spatially varying attenuation distribution decomposes as
\begin{equation}\label{eq:mu-ansatz}
  \mu(\ee,\cdot) = \sum_{m=1}^M \mu_m(\ee)\,\signal[m],
\end{equation}
where $\signal[m] \in \R^{N}$ are the discrete material density maps, and $\mu_m(\ee)$ are known tabulated spectral attenuation coefficients for the $m$th material, $m=1,\dots,M$. Assuming a finite number of rays and a finite number of effective spectra, and discretizing the integral over the spectral variable, yields the MSCT measurements
\begin{align}\label{eq:data1}
  \data[b]_{p}
  &\coloneqq \Phi_b\!\bigl((\radon \signal)_{p}\bigr), \\
  \Phi_b(z_p)
  &\coloneqq \sum_{i=1}^E s_{b,i}\,
    \exp\!\Bigl(- \sum_{m=1}^M \mu_{m,i}\, z[m]_p \Bigr),
\end{align}
for all $(p,b) \in \{1,\dots,P\} \times \{1,\dots,B\}$. Here $ (s_{b,i})_{i=1}^E$ are discretized effective spectra,  $\radon$ is applied separately  to each channel of $\signal \in \R^{N \times M}$, and $\Phi \colon \R^M \to \R^B$ is applied independently to each ray measurement, with $z_p = (\radon \signal)_p \in \R^M$, along the channel dimension. After spectral normalization, we assume $\sum_{i=1}^E s_{b,i} = 1$ for all $b$. Accounting for noise, we obtain the inverse problem \eqref{eq:ip} with $\forward = \Phi \circ \radon \colon \R^{N \times M} \to \R^{P \times B}$, resulting in an inherently nonlinear forward model.

\item \textsc{One-step reconstruction (CP-fast):}
In \cite{prohaszka2024derivative}, we propose a simple yet effective iterative algorithm for MSCT image reconstruction. With
$\Mo \coloneq (\mu_{i,m})_{i,m} \in \R^{E \times M}$ and
$\So \coloneq (s_{b,i})_{b,i} \in \R^{B \times E}$ denoting the matrices of discretized material attenuations and effective spectra, the MSCT forward map is written as 
\begin{align}\label{eq:forward}
  \forward(\signal)
  &= \Phi(\radon \signal),
  \quad \text{with}\quad
  \Phi(z)
  \coloneqq \eexp\!\bigl(-\, z\, \Mo^\top \bigr)\, \So^\top,
\end{align}
where $\eexp$ denotes the pointwise exponential.  We further write $\derivative^\ddagger$ for the pseudoinverse of $\derivative = \So \Mo$. The CP-fast algorithm of \cite{prohaszka2024derivative} then reads
\begin{equation}\label{eq:CP-fast}
  \signal^{(k+1)}
  = \signal^{(k)}
    - s_k \,\radon^{\top}\,
      \bigl(\llog \forward(\signal^{(k)}) - \llog \data\bigr)\,
      (\derivative^\ddagger)^{\top},
\end{equation}
with initialization $\signal^{(0)} = 0$ and step sizes $s_k$, where $\llog$ denotes the pointwise logarithm and $\radon^{\top}$ the adjoint of the (discrete) Radon transform. The strength of \eqref{eq:CP-fast} is that it neither requires computation of the derivative of $\forward$ nor inversion related to the full forward map, and it contains a preconditioner $(\derivative^\ddagger)^{\top} \in \R^{B \times M}$ that is of small size and can be precomputed easily. Despite its simplicity, CP-fast has been demonstrated in \cite{prohaszka2024derivative} to converge rapidly with numerically cheap iteration steps, outperforming tested one-step algorithms compared in \cite{Mory_2018}.

\item \textsc{Learned MSCT reconstruction:}
Let us also  comment on existing work on learning for MSCT. Early works focused on supervised methods, including \cite{clark2018multi,gong2020deep,abascal2021material,wu2021deep}. More closely related to SPLIT are recently proposed self-supervised approaches for MSCT such as \cite{fang2021iterative,chen2023improving,inkinen2022unsupervised,ji2024senas,he2022spectral2spectral,kumrular2024unsupervised}. However, all these methods are based on two-step reconstruction pipelines that separate image formation and material decomposition. In contrast, we target joint image formation and material decomposition in a single step, by combining one-step iterative image reconstruction with self-supervised learning. To the best of our knowledge, no self-supervised method in such a joint setting has been proposed.   

\end{myitem}

\subsection{Double-SPLIT for MSCT}
\label{sec:double-split}

The realization of SPLIT for MSCT requires specifying the data partitions, and the choice of reconstruction maps. We focus here on  the Double-SPLIT version using two data partitions. The Single-SPLIT version can be derived as a simplification of Double-SPLIT  as indicated below.

\begin{myitem}

\item \textsc{Dual-Partitioning:}
We split the measurements $\data$ into two partitions along the spatial detector index $P$, and apply the same partitioning to all energy bins. We reshape data as $\data \in \R^{P_1 \times P_2 \times M}$ where the first index is the angular direction, the second is the offset direction in the sinograms, and the third is the unchanged energy-bin coordinate. We then take partitions $\mathcal{P}_1 = \{\Omega_{1,1}, \Omega_{1,2}\}$ and $\mathcal{P}_2 = \{\Omega_{2,1}, \Omega_{2,2}\}$ with 
\begin{align*}
    \Omega_{1,1} 
    &\coloneq  
    ((2\mathbb{Z}+1) \cap \{1,2,\dots,P_1\})
    \times \{1,2,\dots,P_2\}
    \times \{1,2,\dots,M\} \,,   
    \\
    \Omega_{1,2} 
    &\coloneq  
    ((2\mathbb{Z}) \cap \{1,2,\dots,P_1\})
    \times \{1,2,\dots,P_2\}
    \times \{1,2,\dots,M\} \,,
    \\
    \Omega_{2,1} 
    &\coloneq  
    \{1,2,\dots,P_1\}
    \times ((2\mathbb{Z}+1) \cap \{1,2,\dots,P_2\})    
    \times \{1,2,\dots,M\} \,,
    \\
    \Omega_{2,2} 
    &\coloneq  
    \{1,2,\dots,P_1\} \times ((2\mathbb{Z}) \cap \{1,2,\dots,P_2\})
    \times \{1,2,\dots,M\} \,,
\end{align*}
which yields data splits in angular and projection direction, respectively     
\begin{align} \label{eq:double11}
    \data_{1,1}(i,j,m) 
    &\coloneq  \data(i,2j,m) \,,
    \\ \label{eq:double12}
    \data_{1,2}(i,j,m) 
    &\coloneq  \data(i,2j - 1,m) \,,
     \\ \label{eq:double21}
    \data_{2,1}(i,j,m) &\coloneq \data(2i,j,m) \,,
    \\ \label{eq:double22}
     \data_{2,2}(i,j,m) &\coloneq \data(2i - 1, j,m) \,.    
\end{align}

\item \textsc{Partial reconstructions:}
Partial initial reconstructions are all defined by $K$ steps of the CP-fast algorithm \eqref{eq:CP-fast}, 
\begin{equation} \label{eq:doublesplt-rec}
\rec_{i,j} \data_{i,j} = \signal^{(K)} \,.
\end{equation}
For the angular split data \eqref{eq:double11}, \eqref{eq:double12} we actually apply \eqref{eq:CP-fast} with the restricted data \eqref{eq:double11}, \eqref{eq:double12} and angular-restricted forward maps $\forward_{1,1}$ and $\forward_{1,2}$, respectively, instead of $\forward$. For the offset-direction split data \eqref{eq:double21}, \eqref{eq:double22} we first apply linear interpolation to $\data_{2,1}$ and $\data_{2,2}$ and then apply \eqref{eq:CP-fast} with the full forward operator $\forward$ to the interpolated data.

\item \textsc{Training and inference:}
To all $\rec_{i,j} \data_{i,j}$ we apply $f \colon \R^{N \times B} \to \R^{N \times M}$ which is defined by a U-net $\Phi \colon \R^{N} \to \R^{N}$ \cite{ronneberger2015u} taken the same in each channel, $f(\signal) = \big( \Phi \signal[m] \big)_m$. The U-net network is trained jointly
in data space by minimizing  
\begin{equation} \label{eq:doublesplit}
\loss_Y^2(\Phi) = 
\frac{1}{2}
\sum_{i,j=1}^2
\sum_{m=1}^M
\left\|
\mathcal{A}_{i,j}\, \Phi_\theta\!\big( (\rec_{i,j}^c \data_{i,j}^c)[m] \big) - \data_{i,j}[m]
\right\|^2 \,.
\end{equation}
Inference is done by \eqref{eq:multisplit-inf}, which in the above situation becomes
\begin{equation} \label{eq:doublesplit-inf}
\rec(\data)[m] 
\coloneq 
\frac{1}{4} 
\sum_{i,j=1}^2
\Phi_\theta \bigl( (\rec_{i,j} \data_{i,j})[m] \bigr) \,.
\end{equation}

\item \textsc{Earlystopping:}
The loss \eqref{eq:doublesplit} is data-driven and exploits the available samples. Our proposed approach involves comparing $\Phi_\theta (\rec (\data_{i,2}))$ with $\Phi_\theta (\rec (\data_{i,1}))$ in terms of the Peak Signal-to-Noise Ratio (PSNR).
More precisely, we define an early stopping metric as the average PSNR
\begin{equation} \label{eq:stop}
\overline{\psnr} \coloneqq \frac{1}{2M} \sum_{i=1}^{2} \sum_{m=1}^M \psnr \big(\Phi_\theta ((\rec \data_{i,2})[m]), \Phi_\theta (( \rec \data_{i,1})[m])\big).
\end{equation}
We terminate the training when this metric is minimized. 
The underlying rationale is rooted in the dynamics of model overfitting: as the model begins to overfit, its outputs for different target domains start to diverge. Specifically, since $\Phi_\theta (\rec (\data_{1,2}))$ and $\Phi_\theta (\rec (\data_{1,1}))$ are trained on disjoint targets, their outputs become increasingly dissimilar after overfitting; a similar divergence occurs with $\Phi_\theta (\rec (\data_{2,1}))$ and $\Phi_\theta (\rec (\data_{2,2}))$. By monitoring these cross-domain output differences through the $\psnr$ metric, we can detect the onset of overfitting.

\end{myitem}

\subsection{Implementation details}

We next specify details for the results presented, concerning MSCT specifications, the network training and the used data sets.         

\begin{myitem}

\item \textsc{MSCT specifications:}
The numerical settings for the MSCT forward map are taken as in \cite{prohaszka2024derivative}, which uses the implementation of \cite{Mory_2018}. We consider $M = 3$ base materials (water, iodine, gadolinium) and $B = 5$ energy bins. The energy variable is discretized on $150$ uniform nodes between $0$ and $150$ keV. Images use $N = 256 \times 256$ pixels, and the Radon transform yields $P = 262{,}450$ line integrals. Noisy measurements $\data$ are generated using a more realistic forward model with Poisson photon statistics and electronic noise added. For initial reconstructions with the full forward model we use CP-Fast with full data.

\item \textsc{Comparison methods:}
We compare the following  reconstruction methods     
\begin{tritemize}
\item CP-Fast: outlined in Section~\ref{sec:msct} (iterative benchmark).
\item X-space loss: self-supervised (single partition and reconstruction-space loss).
\item Single-SPLIT: self-supervised (single partition and measurement-space loss).
\item Double-SPLIT: outlined in Section~\ref{sec:double-split} (two partitions and measurement-space loss).
\end{tritemize}
Here the X-space loss method and Single-SPLIT use the single angular partition $\mathcal{P}_1$ and the reconstructions $\rec_{1,1}$ and $\rec_{1,2}$, with the losses, respectively,
\begin{align*}
\loss_X^1(\Phi) &= 
\sum_{j=1}^2
\sum_{m=1}^M
\left\|
 \Phi_\theta\!\big( (\rec_{1,j}^c \data_{1,j}^c)[m] \big) - (\rec_{1,j} \data_{1,j})[m]
\right\|^2,
\\
\loss_Y^1(\Phi) &= 
\sum_{j=1}^2
\sum_{m=1}^M
\left\|
\forward_{1,j}\, \Phi_\theta\!\big( (\rec_{1,j}^c \data_{1,j}^c)[m] \big) - \data_{1,j}[m]
\right\|^2 \,.
\end{align*}
Inference in both cases uses the average
\[
\rec(\data)[m] \coloneq \frac{1}{2} \sum_{j=1}^2
\Phi_\theta\!\big( (\rec_{1,j}^c \data_{1,j}^c)[m] \big).
\]
Note that the X-space loss method can be seen as an adaptation of Noise2Inverse to non-linear tomography, while  Single-SPLIT can be seen as an adaptation of  Sparse2Inverse to non-linear tomography.

\item \textsc{Dataset:}
We generate a synthetic MSCT dataset comprising $100$ $3$-channel material images $\signal$ and corresponding $5$-channel multispectral measurements $\data = \forward(\signal) + \noise$. Each phantom is derived from a randomly deformed, modified Shepp--Logan phantom, with phantom regions assigned to the material channels. To ensure physically meaningful and numerically stable attenuation modeling, the material coefficients were scaled to reflect realistic concentration levels observed in clinical CT acquisitions. In particular, the iodine and gadolinium channels were normalized to water-equivalent attenuation units and further scaled to represent typical concentrations on the order of a few percent relative to water, while water serves as the reference material and is left unscaled. This choice is practically motivated and ensures that the relative contributions of all channels remain within a realistic and well-conditioned range. We use $60$ phantoms for training, $20$ for validation, and $20$ for testing. The three channels of an example phantom are shown in the left column of Figure~\ref{fig:msct_725}.

\item \textsc{Training:} For the U-net we use the implementation from \url{https://github.com/milesial/Pytorch-UNet}. Training uses Adam with learning rate $10^{-4}$, batch size $1$, for up to $15{,}000$ epochs. For early stopping we monitor the sum \eqref{eq:stop} of the estimated PSNRs across materials and halt training when this sum attains its maximum. We used the same U-net model for all learned methods.

\end{myitem}

\begin{figure}[htb]
        \includegraphics[width=\textwidth]{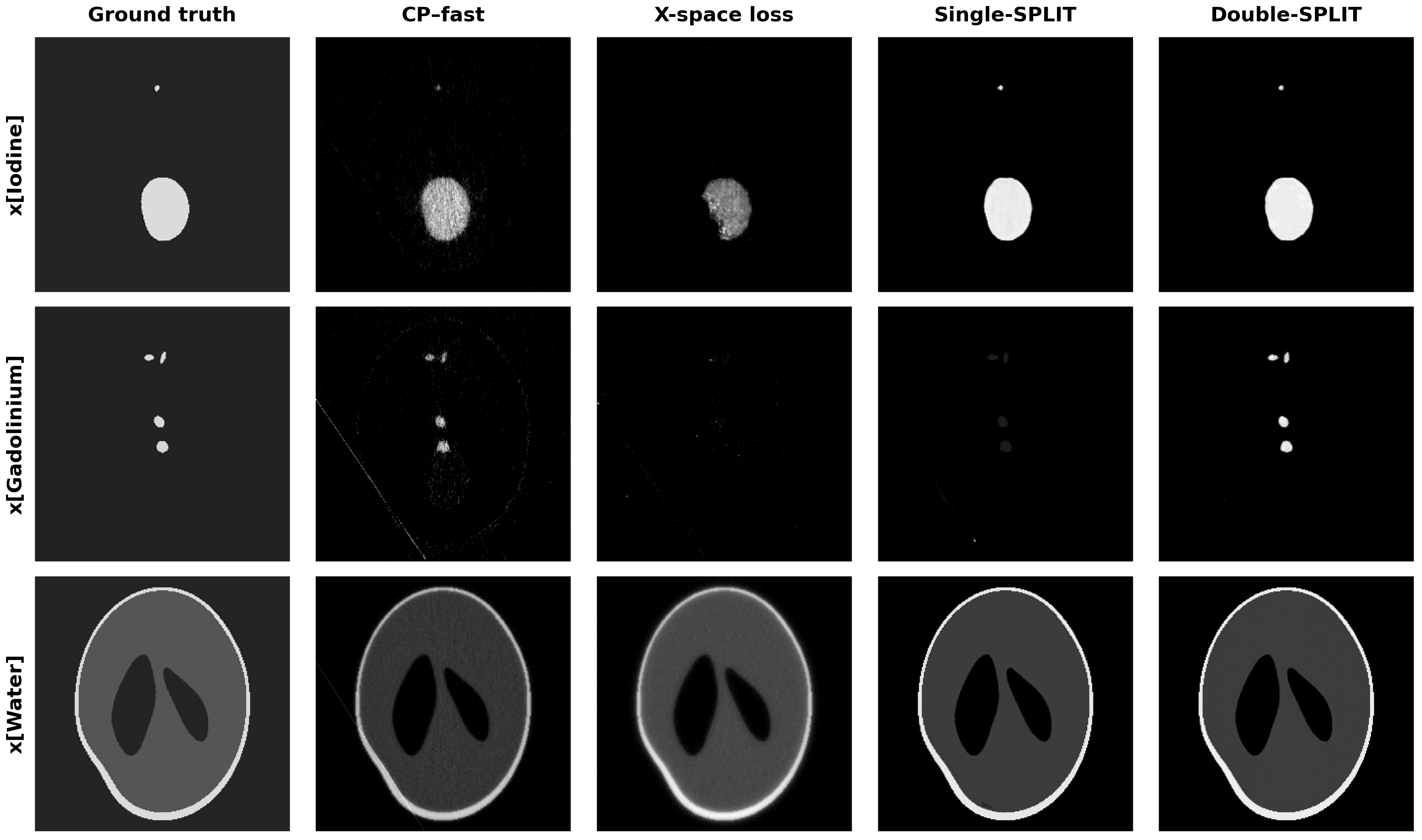}
 \caption{Multispectral CT reconstructions from $N_\theta = 725$ projection angles. Rows: iodine (top), gadolinium (middle), water (bottom). Columns (left to right): ground truth, CP-Fast, N2I, Single-SPLIT, Double-SPLIT.}
  \label{fig:msct_725}
\end{figure}

\begin{figure}[htb]
\includegraphics[width=1\textwidth]{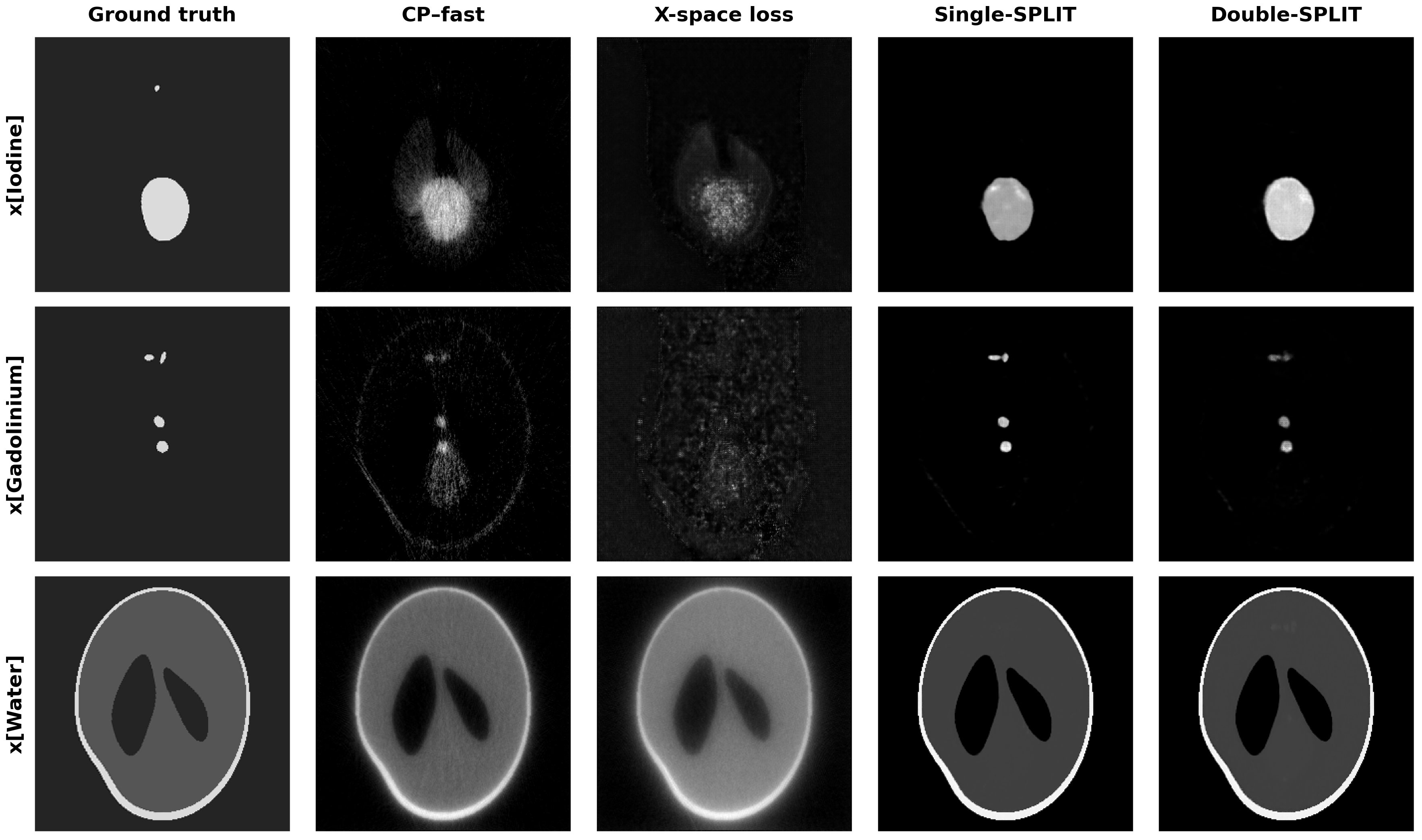}
    
  \caption{Multispectral CT reconstructions from $N_\theta = 145$ projection angles. Rows and columns as in Fig.~\ref{fig:msct_725}.}
  \label{fig:msct_145}
\end{figure}

\begin{figure}[htb] 
        \includegraphics[width=1\textwidth]{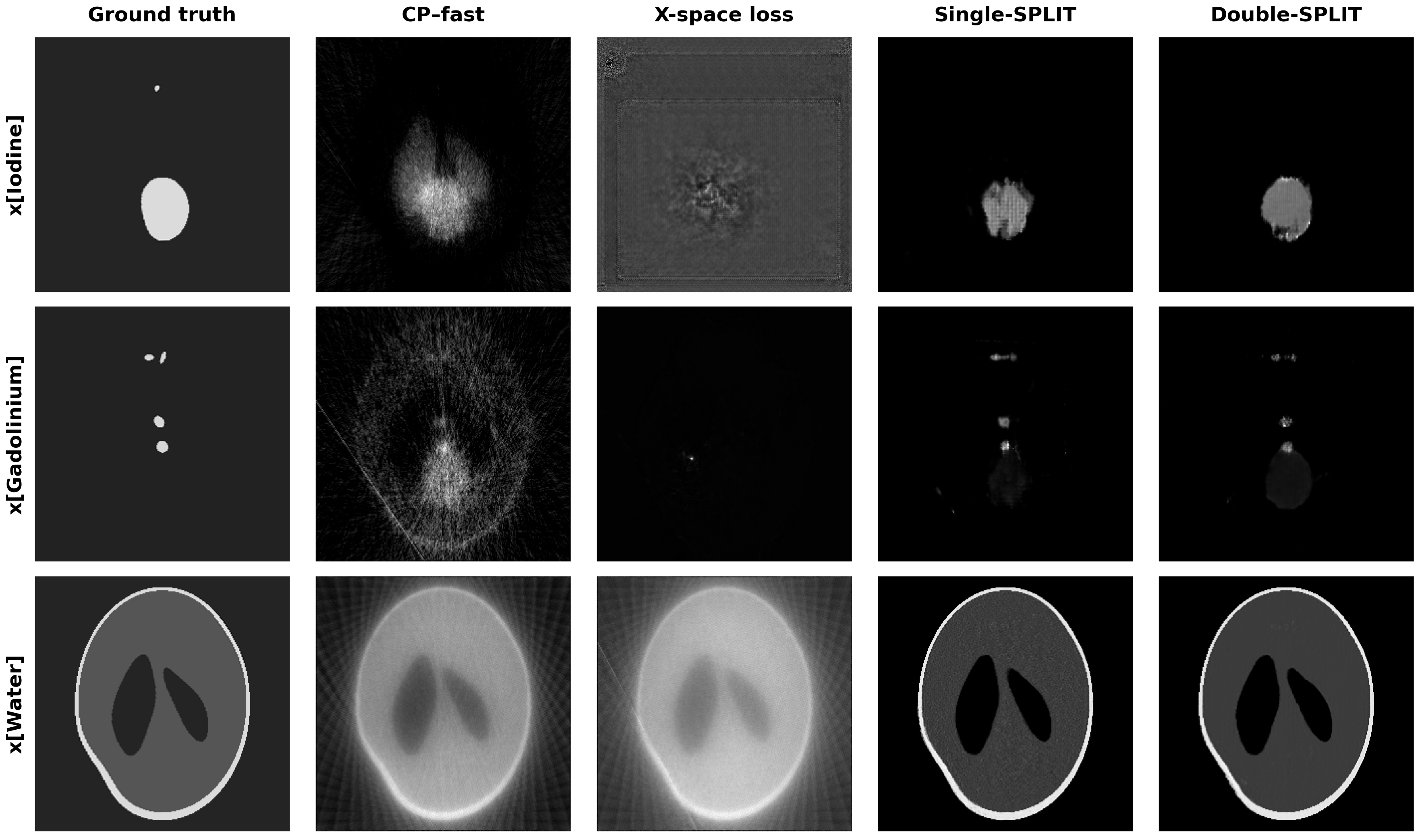}
     \caption{Multispectral CT reconstructions from $N_\theta = 29$ projection angles. Rows and columns as in Fig.~\ref{fig:msct_725}.}    
    \label{fig:msct_29} 
\end{figure}

\subsection{Results}

Next, we present reconstruction results across three levels of angular sampling: $N_\theta \in \{725,145,29\}$ views. Visual results are shown in Figs.~\ref{fig:msct_725}, \ref{fig:msct_145}, and \ref{fig:msct_29}. PSNR and SSIM for each method and view count are reported in Table~\ref{tab:metric}, with values given separately for each material. We observe that Single-SPLIT and Double-SPLIT consistently outperform the iterative CP-Fast baseline and the self-supervised X-space loss method across metrics and materials. Under ultra-sparse sampling ($29$ views), Double-SPLIT performs best overall, whereas for denser sampling ($725$ or $145$ views) Single-SPLIT is slightly superior.

\begin{table}[htb]
  \caption{PSNR and SSIM for three materials at different view counts (higher is better). PSNR is reported in dB. Best value per material and view count is in bold.}
  \label{tab:metric}
  \centering
  \footnotesize
  \setlength{\tabcolsep}{5pt}
  \renewcommand{\arraystretch}{1.15}
  \begin{tabular}{
    @{} c l
    S[table-format=2.2] S[table-format=2.2] S[table-format=2.2]
    S[table-format=1.2] S[table-format=1.2] S[table-format=1.2] @{}
  }
    \toprule
    & & \multicolumn{3}{c}{PSNR [dB]} & \multicolumn{3}{c}{SSIM} \\
    \cmidrule(lr){3-5} \cmidrule(lr){6-8}
    \multicolumn{1}{c}{Views} & \multicolumn{1}{c}{Method}
      & \multicolumn{1}{c}{Iodine} & \multicolumn{1}{c}{Gadolinium} & \multicolumn{1}{c}{Water}
      & \multicolumn{1}{c}{Iodine} & \multicolumn{1}{c}{Gadolinium} & \multicolumn{1}{c}{Water} \\
    \midrule
    \multirow{3}{*}{725}
      & X-space loss  (Reference)    & 24.41 & 25.69 & 24.72 & 0.94 & 0.81 & 0.85\\
      & Single-SPLIT (Proposed)  & 35.33 & 28.06 & \bfseries 39.60 & 0.99 & 0.98 & \bfseries 1.00 \\
      & Double-SPLIT (Proposed)
                              & \bfseries 35.51 & \bfseries 37.75 & 37.83
                              & \bfseries 0.99 & \bfseries 0.99 &  0.99 \\
    \midrule
    \multirow{3}{*}{145}
      & X-space loss  (Reference)   & 20.13 & 21.42 & 19.78 & 0.10 & 0.02 & 0.52 \\
      & Single-SPLIT (Proposed) & 30.50 & \bfseries 33.18 & \bfseries 36.23
                              & \bfseries 0.98 & \bfseries 0.98 & \bfseries 0.99 \\
      & Double-SPLIT (Proposed)
                              & \bfseries 31.66 & 32.96 & 35.96
                              & 0.97 & 0.95 & \bfseries 0.99 \\
    \midrule
    \multirow{3}{*}{29}
      & X-space loss (Reference)   & 15.99 & 13.16 & 17.31 & 0.01 & 0.00 & 0.33 \\
      & Single-SPLIT (Proposed)& 23.93 & \bfseries 29.30 & 31.88
                              & 0.95 & \bfseries 0.94 & 0.88 \\
      & Double-SPLIT (Proposed)
                              & \bfseries 25.81 & 26.98 & \bfseries 33.53
                              & \bfseries 0.97 & 0.93 & \bfseries 0.98 \\
    \bottomrule
  \end{tabular}
\end{table}

\section{Conclusion}

Learned image reconstruction has become state of the art across many imaging modalities. Of particular interest are self-supervised methods that require only measurement samples $y = \forward(x) + \noise$ without access to paired data $(x, y)$ or even ground-truth images $x$. While such approaches have primarily focused on linear problems, this work addresses the nonlinear, incomplete, and noisy regime relevant to nonlinear inverse problems. We introduced and realized SPLIT (Self-supervised Partitioning for Learned Inversion in Nonlinear Tomography) for one-step MSCT reconstruction and material decomposition. Our framework leverages measurement-domain losses, cross-partition consistency, and complementary information across spectral channels, together with an automatic stopping criterion that serves as regularization. SPLIT consistently improves reconstruction quality and robustness compared with  iterative reconstruction baseline and a self-supervised baseline using a reconstruction-domain loss.   

Future work will include evaluation on real MSCT datasets with realistic calibration and model mismatch, systematic studies of noise (e.g., Poisson and electronic noise) and sparsity regimes, and theoretical analysis of identifiability, stability, and convergence under nonlinear forward models. We also plan to integrate recent advances in  physics-informed priors, quantify uncertainty, and extend the framework to other nonlinear and ill-posed inverse problems such as quantitative photoacoustic tomography and inverse diffusion (diffuse optical tomography).

\bibliographystyle{unsrt}
\bibliography{refs}

\end{document}